\documentclass{article}

\usepackage{amssymb}

\newtheorem{theorem}{Theorem}

\newtheorem{lemma}[theorem]{Lemma}
\newtheorem{definition}[theorem]{Definition}
\newtheorem{proposition}[theorem]{Proposition}

\newenvironment{proof}{\noindent\bf{Proof.}\rm}{\hfill$\blacksquare$\bigskip}

\begin{document}

\title{Why are images smooth?}

\author{Uriel Feige~\thanks{Department of Computer Science and Applied Mathematics, the Weizmann Institute, Rehovot, Israel.
{\tt uriel.feige@weizmann.ac.il}. The author holds the Lawrence G. Horowitz Professorial Chair at the Weizmann Institute.
Work supported in part by the Israel Science Foundation
(grant No. 621/12) and by the I-CORE Program of the Planning and Budgeting Committee and The Israel Science Foundation (grant No. 4/11).}}

\maketitle

\begin{abstract}
It is a well observed phenomenon that natural images are smooth, in the sense that nearby pixels tend to have similar values. We describe a mathematical model of images that makes no assumptions on the nature of the environment that images depict. It only assumes that images can be taken at different scales (zoom levels). We provide quantitative bounds on the smoothness of a typical image in our model, as a function of the number of available scales. These bounds can serve as a baseline against which to compare the observed smoothness of natural images.
\end{abstract}

\section{Introduction}

An {\em image} is a two dimensional array of pixels with $n$ rows and $m$ columns (typically we will take $m=n$), where a pixel $p$ has a real value $x_p \in [0,1]$. (This naturally corresponds to a grey-scale image, though the results extend in a straightforward way to color images, by applying them separately to each of the the basic colors RGB).
It is common wisdom that in natural images nearby pixels tend to have similar values. One may refer to this property as saying that natural images are {\em smooth}. Several hypotheses can be made as to why natural images are smooth. For example:

\begin{enumerate}

\item Our physical world has the property that environments are smooth, and images merely reflect this physical reality.

\item Physical and technological constraints in generating images (for example, properties of lenses) tend to create smooth images, regardless of whether the environment is smooth or not.

\item There is a {\em selection bias} - the portions of the environment that we tend to depict in images are the smooth portions.

\end{enumerate}

In order to test such hypotheses, it is desirable to compare them against a null hypothesis. One baseline for comparison is that of random arrays of pixels. However, we propose a different baseline for comparisons, that we shall refer to as {\em images} (in distinction from {\em natural images}).

We study a formal mathematical model of images that assumes that there are no technological constraints in depicting images of the environment, and assumes that there is no selection bias -- any portion of the environment is equally likely to be depicted. We show that in our formal model, some level of smoothness of images is to be expected, regardless of any assumptions on the physical environment that is being depicted.

The key aspect that our model makes use of is that environments are depicted in various scales. For example, our eyes may focus on objects as small as a few centimeters in length (say, an insect), or sceneries spanning many kilometers (say, a distant mountain range). It is common wisdom that the smoothness of an object depends on the scale at which it is depicted. Consider for example a very large black and white checkerboard pattern. Viewed from a large distance, one pixel in the image will average the value of many checkerboard squares, and hence the image may be uniformly grey (very smooth). Viewed from a very short distance, every square may correspond to many pixels, and then nearby pixels will have the same value, so the image will be very smooth almost everywhere (except on the boundary between squares). However, at some intermediate scale, each square will occupy a small number of pixels (say one pixel, or four pixels), and then adjacent pixels will have very different values and the image will not be considered smooth.

In our study we present a formal model, and within this model we provide quantitative results regarding the effect that having multiple scales has on the typical smoothness of images. Our results imply that a nontrivial level of smoothness of images should be attributed to some universal mathematical principles that have nothing to do with the environment that is being depicted.

\subsection{Related work}

There is a vast body of work on {\em natural image statistics} (see~\cite{book}, for example). Smoothness is a well observed aspect of these statistical properties.  Moreover, natural images tend to have interesting and useful statistical properties that go much beyond smoothness (see~\cite{irani}, for example). A key aspect of our study is that environments are depicted in various scales. This same aspect appears in existing studies of natural images (see~\cite{forest, ruderman, AGM1999, mumford}, for example), though the focus of work in these references is different from ours: it relates to observed scale invariant properties of natural images, and to statistical models that attempt to explain this phenomena. Our current work does not deal with natural images, but rather with images in some abstract mathematical model. Our results can be contrasted against known results on natural images, but do not directly provide new information about natural images.

The techniques used in our proofs are of the form often used in image processing literature and practice. They are strongly related to a wavelet transform~\cite{wavelet} with a Haar basis.

As our results deal with abstract notions of images rather than natural images, the mathematical principles that underlie them are applicable in other settings, and in fact similar principles were used in other settings. Specifically, Theorem~\ref{thm:local} is a variation on a certain {\em local repetition lemma} proved in~\cite{FKT} in the context of sequential decision making, and Proposition~\ref{pro:sqrt} is based on an example given in~\cite{FKT} showing the tightness of the parameters in the local repetition lemma.

\section{A formal model of images}
\label{sec:definitions}

An {\em image} $I_{n,m}$ is a two dimensional array of pixels with $n$ rows and $m$ columns. We shall sometimes omit the subscripts $n,m$ and simply use $I$. A pixel $p$ has a real value $x_p \in [0,1]$.   Numbering the pixels in $I_{n,m}$ by $(i,j)$ with $0 \le i < n$ and $0 \le j < m$, two pixels $(i_1,j_1)$ and $(i_2,j_2)$ are {\em adjacent} if either $i_1=i_2$ and $|j_1 - j_2| = 1$, or $j_1 = j_2$ and $|i_1-i_2| = 1$. Borrowing standard graph theoretic terminology, we refer to a pair of adjacent pixels as an {\em edge} in the image, and we denote the set of edges in the image by $E$. It is not difficult to verify that $|E| = 2nm - n - m$. We remark that the notation $E$ would also be used in order to denote the expectation operator, but the intended use of the notation $E$ (either as set of edges or as expectation) will be clear from the context.

The {\em discrepancy} of two pixels $p$ and $q$ is a measure of how different their value is. We consider two different ways of measuring discrepancy, {\em linear discrepancy} $D_1(p,q) = |x_p - x_q|$ and {\em quadratic discrepancy} $D_2 = (x_p-x_q)^2$. The subscript of $D$ indicates the power to which $|x_p - x_q|$ is raised. $D_1$ is perhaps the more natural of these two measures, but it is mathematically more convenient to work with $D_2$.

\begin{definition}
\label{def:discrepancy}
The {\em local discrepancy} of an image $I_{n,m}$ is the average discrepancy for pairs of adjacent pixels. It is denoted by $LD_1(I) = \frac{1}{|E|}\sum_{(p,q)\in E}  |x_p - x_q|$ and $LD_2(I) = \frac{1}{|E|}\sum_{(p,q)\in E}  (x_p - x_q)^2$. The {\em global discrepancy} of an image $I$ is the average discrepancy over all pairs of pixels whether adjacent or not, including also pairs in which $p$ and $q$ are the same pixel. It is denoted by $GD_1(I) = \frac{1}{n^2m^2}\sum_{p \in I; q \in I}  |x_p - x_q|$ and $GD_2(I) = \frac{1}{n^2m^2}\sum_{p \in I; q \in I}  (x_p - x_q)^2$. In cases where we do not wish to distinguish between linear and quadratic discrepancy, we shall use the notation $LD(I)$ and $GD(I)$ with no subscript.
\end{definition}

The range of possible values of local and global discrepancy is as specified in the following Proposition.

\begin{proposition}
\label{pro:simple}
For every image $I$ the following hold: $0 \le LD(I) \le 1$ and $0 \le GD(I) \le \frac{1}{2}$.
\end{proposition}

\begin{proof}
Nonnegativity follows immediately from Definition~\ref{def:discrepancy}.

$LD(I) \le 1$ holds because for every pixel $p$, $0 \le x_p \le 1$. In a checkerboard pattern with pixel values alternating between~0 and~1 the bound $LD(I) = 1$ is attained. On the same pattern, the bound $GD(I) = \frac{1}{2}$ is attained (if $n$ is even). Convexity of the functions $|x|$ and $x^2$ implies that to maximize $GD(I)$ one needs $x_p \in \{0,1\}$ for every $p$, and one needs the number of 0-pixels to be equal to the number of 1-pixels. In this extreme case $GD(I) = \frac{1}{2}$.
\end{proof}

An image may be smooth in several different senses, and we shall explicitly distinguish between them. One sense of being smooth is that of having low local discrepancy. A consequence of this smoothness is that the image can be compressed: traversing all pixels via some connected path (e.g., row by row in a snakelike fashion), for every new pixel we encounter we already have some prior estimate on its value, based on the pixel preceding it. Another sense of being smooth is by having low global discrepancy. This is a stronger notion than low local discrepancy, due to the following proposition.

\begin{proposition}
\label{pro:half}
For every $n$ by $n$ image $I$, $GD(I) \ge \frac{1}{2}LD(I) - O(\frac{1}{n})$.
\end{proposition}

\begin{proof}
Pick a random pixel $p$, and independently, a random pixel $q$ and a random neighbor $q'$ of $q$. By the triangle inequality, $|x_p - x_q| + |x_p - x_{q'}| \ge |x_q - x_{q'}|$. Observe that $GD_1(I)$ exactly equals the expectation of $|x_p - x_q|$, and nearly equals the expectation of $|x_p - x_{q'}|$ (up to an $O(1/n)$ term that is the result of boundary effects). Likewise, $LD_1(I)$ nearly equals the expectation of $|x_q - x_{q'}|$ (up to an $O(1/n)$ term that is the result of boundary effects). Hence averaging over all choices of $p,q,q'$ the inequality $GD_1(I) \ge \frac{1}{2}LD_1(I) - O(\frac{1}{n})$ is proved.

A simple modification to the proof above shows that $GD_2(I) \ge \frac{1}{4}LD_2(I) - O(\frac{1}{n})$. The proof of the stronger claim that $GD_2(I) \ge \frac{1}{2}LD_2(I) - O(\frac{1}{n})$ is deferred to Section~\ref{sec:preliminary}.
\end{proof}

Yet another sense of being smooth is by having a high local correlation coefficient.

\begin{definition}
\label{def:LGR}
The {\em local correlation} of an image $I$ is $LC(I) = \frac{GD_2(I)}{LD_2(I)}$, where $\frac{0}{0}$ is interpreted as being equal to~1.
(Observe that $LD_2(I) = 0$ if and only if $GD_2(I) = 0$.)
\end{definition}

Observe that for an image  $I$ in which the values of pixels are chosen as independent identically distributed (i.i.d.) random variables, one would expect $LC(I) \simeq 1$. An LC value that significantly deviates from~1 is an indication that the image is not just a collection of random pixels, but rather that there are local correlations. High local correlation (LC values larger than~1) relates to the experience of putting together a jigsaw puzzle: it is a good heuristic to try to match together jigsaw pieces of roughly the same color, rather than just trying to match together random pieces. This is because local discrepancy is typically smaller than global discrepancy.

The main claim of this manuscript is that most images are smooth to a noticeable extent. However, the definitions that we gave so far point to the contrary. If an image is just an array of pixels, then a natural interpretation of the term {\em most} is to select the values of these pixels at random in an i.i.d. fashion, with each pixel value distributed uniformly in the range $[0,1]$. This will give LC value of roughly~1 which we do not consider as smooth, and also the local discrepancy would not be low (one expects $LD_1(I) \simeq \frac{1}{3}$ in this case, details omitted).

To be able to substantiate a claim of smoothness, we refine the definition of what an image is. This will lead to natural probability distributions over images that are different from the uniform one stated above, and with respect to these probability distributions most images will be smooth.

\subsection{The probability distribution $D$ over images}

An image, unlike an arbitrary array of pixels, is meant to be an image of ``something". That it, we assume that there is some underlying {\em environment}, and images depict portions of the environment. We shall not make any assumptions about the environment -- it can be arbitrarily complex and random looking. However, we shall make one assumption about images, and this is that the portions of the environment that images depict can be of different sizes. We now present our model more formally.

There is an environment $U$, which is an $N$ by $N$ grid of {\em cells}. For example, the environment can be a large geographical region (say, of size 100km by 100km), and a cell can be of size corresponding to the smallest unit realistically observable by optical means (say, of side-length $10^{-5}$ meters). In the case described above, $N = 10^{10} \simeq 2^{33}$.
Each cell $c$ has {\em intensity} $I_c \in [0,1]$.

Recall that we defined an image to be an $n$ by $m$ two dimensional array
composed of pixels. To simplify of the rest of the presentation, we shall assume that $m=n$. We require $n$ to be considerably smaller than $N$. For example, for images with 4 mega-pixels, $n \simeq 2^{11}$.

In terms of terminology, the terms {\em grid}, {\em cell}, {\em intensity} and $N$ will be associated with environments, whereas the terms {\em array}, {\em pixel}, {\em value} and $n$ will be associated with images.

There is a {\em scale} associated with an image, which is an integer in the range $[0, k-1]$, where $k$ is some fixed integer satisfying $n2^k \le N$. An image with scale $\ell$ describes an $n2^{\ell}$ by $n2^{\ell}$ portion of $U$, where every pixel of the image corresponds to a square of $2^{\ell}$ by $2^{\ell}$ cells of $U$. The value $x_p$ of a pixel $p$ is the average intensity of the cells that it represents, namely, $x_p = 2^{-2\ell} \sum_{c\in p} I_c$. One may think of an image as an $n$-pixel by $n$-pixel photograph of some portion of $U$, taken at a zoom level determined by $\ell$. (This is not meant to be a model that incorporates all optical and technological constraints when describing what a photograph is, but merely a simple approximate model.) Pixels of highest resolution ($\ell = 0$) in our model correspond to single cells in the environment. This convention simplifies the presentation without significantly affecting our results.

Now we describe our probability distribution $D$ that governs
which portion of $U$ is contained in the image. This involves two aspect. One is the {\em scale} of the image: in $D$ the scale is an integer $\ell$ chosen uniformly at random in the range $[0,k-1]$. The other aspect is the {\em location} of the image within $U$. In $D$ the location is a cell $(i,j)$ chosen uniformly at random in the range $0 \le i,j \le N-1$, and the image extends over those cells $(i',j')$ with $i' - i < n2^{\ell}$ modulo $N$ and $j' - j < n2^{\ell}$ modulo $N$. Observe that under this definition, an image that is close to the boundary of $U$, ``wraps around" and continues at the other side of $U$. Hence $U$ is treated as a torus rather than as a grid. This is done for technical reasons, so as not to complicate the analysis by boundary effects. It has very little influence on the end results, because a random image is unlikely to be at the boundary of $U$, and even if it is, only $O(n)$ out of its $n^2$ pixels are at the boundary of $U$.

Observe that under the distribution $D$, every cell of $U$ is equally likely to be part of an image. Each cell of $U$ belongs to at most one pixel in the image, but each pixel in the image in scale $\ell$ contains $2^{2\ell}$ cells. Observe also that $D$ as described above is simply the uniform distribution over all possible images (portions of $U$ that satisfy the size constraints of images).

\begin{definition}
\label{def:DU}
Given an $N$ by $N$ environment $U$ and integers $n,k \ge 2$ satisfying $n2^k < N$, the average local discrepancy of the $U$, denoted by $LD(U,k)$, is the expected local discrepancy of an image sampled from $U$ according to distribution $D$.  Namely:

$$LD(U,k) = E_{I \leftarrow D} LD(I)$$

Analogously,  the average global discrepancy of $U$ is

$$GD(U,k,n) = E_{I \leftarrow D} GD(I)$$

\end{definition}

Observe that given $U$ and $k$, the average local discrepancy is independent on $n$, but global discrepancy does depend on $n$.


\section{Results}

In this section, the terminology used is as defined in Section~\ref{sec:definitions}. In particular, $k$ is the number of scales, and we always assume that $n2^k < N$. For simplicity, we shall assume that $n$ is a power of~2. Throughout, all logarithms are in base~2. Subscripts of~1 or~2 following $D$ denote whether we are referring to linear or quadratic discrepancy.

\begin{theorem}
\label{thm:local}
For every environment $U$ its average local discrepancy satisfies  $LD_2(U,k) \le \frac{1}{k}$.
\end{theorem}

\begin{proposition}
\label{pro:global}
When $\log n \le k$, there are environment $U$ for which the average global discrepancy satisfies $GD_2(U,k,n) \ge \frac{\log n}{2k}$, up to low order terms that tend to~0 as $n$ grows.
\end{proposition}

Let us contrast Theorem~\ref{thm:local} with Proposition~\ref{pro:global}. Suppose that images can correspond to objects as small as one centimeter in the environment $U$ (say, a photo of an insect), up to objects as large as ten kilometers (say, a photo of a landscape). This gives $k \simeq \log 10^6 \simeq 20$ different scales for images. Suppose that every image has 1024 by 1024 pixels. Then $n = 2^{10}$ and $\log n = 10$. Proposition~\ref{pro:global} shows that for a random image (sampled from $D$), the average global discrepancy might be as high as $\frac{\log n}{2k} = \frac{1}{4}$. Theorem~\ref{thm:local} shows that the average local discrepancy is at most $\frac{1}{k} = \frac{1}{20}$.

Theorem~\ref{thm:local} concerns quadratic discrepancy and not linear discrepancy. Hence possibly $LD_1(n,k) \simeq \frac{1}{\sqrt{k}}$, even though $LD_2(n,k) \le \frac{1}{k}$.

\begin{proposition}
\label{pro:sqrt}
There is some constant $c > 0$ such that for every $k$, there is an environment $U$ for which $LD_1(U,k) \ge \frac{c}{\sqrt{k}}$.
\end{proposition}

The theme of the next theorem is that unless local correlation (in the sense of Definition~\ref{def:LGR}) is significant, then $O(1/k)$ bounds apply not only to quadratic discrepancy, but also to linear discrepancy.

For an environment $U$ and $0 \le \ell \le k-1$, let $LD(\ell)$ denote the average local discrepancy taken only over images of scale $\ell$, and let $GD(\ell)$ denote the average global discrepancy taken only over images of scale $\ell$.

\begin{theorem}
\label{thm:rho}
For an environment $U$ and $1 < \alpha < \frac{\log n}{2}$, suppose that for every $0 \le \ell \le k-1$, $GD_2(\ell) \le \alpha LD_2(\ell)$. Let $0 < p < 1$ be such that $\frac{1 - p^{\log n}}{1-p} = 2\alpha$. Then $\sum_{\ell = 0}^{k-1} LD_1(\ell) \le \frac{1 + \sqrt{p}}{\sqrt{1-p}}$. In particular, as $n$ grows, the upper bound on $LD_1(U,k)$ tends to $\frac{\sqrt{2\alpha} + \sqrt{2\alpha - 1}}{k}$.
\end{theorem}

\section{Proofs}

\subsection{Some preliminary results}
\label{sec:preliminary}

The following propositions collect some properties of discrepancy.

\begin{proposition}
\label{pro:largegap}
For every image $I$,  $LD_2(I) \ge \Omega(\frac{GD_2(I)}{n^2})$.
There are images with $LD_2(I) \le O\left(\frac{GD_2(I)}{n^2}\right)$.
\end{proposition}

\begin{proof}
One can lower bound $LD_2(I)$ as a function of $GD_2(I)$ by the following procedure for sampling an adjacent pair of pixels. First, sample two pixels $p$ and $q$ uniformly at random (as done for computing $GD_2(I)$). Then follow a canonical path from $p$ to $q$, first going along the row of $p$ until the column of $q$ is reached, and then along the column of $q$ until $p$ is reached. Thereafter, a random adjacent pair of pixels $(u,v)$ along this path is chosen. As the path is at most of length $2n$, the triangle inequality for distances implies that $E[(x_u - x_v)^2] \ge \frac{(x_p - x_q)^2}{4n^2}$ (the worst case is when the value of every two adjacent pixels along the path differs by $\frac{|x_p - x_q|}{2n}$). The above procedure for sampling adjacent pixels distorts the uniform distribution over adjacent pixels, but only to limited extent. A pair of adjacent pixels can increase its probability of being sampled (compared to the uniform probability) by at most a constant factor. Hence also with respect to the uniform distribution over pairs of adjacent pixels we must have  $LD_2(I) \ge \Omega(\frac{GD_2(I)}{n^2})$. (The constants in this proof can be improved by a more careful analysis.)

An example of an image $I$ with $GD_2(I) = \Omega(1)$ and $LD_2(I) = O(\frac{1}{n^2})$ is the following: for every $0 \le i < n$, all pixels in row $i$ have the same value $\frac{i}{n}$.
\end{proof}

For linear discrepancy, the bounds in Proposition~\ref{pro:largegap} should be changed to $LD_1(I) \ge \Omega(\frac{GD_1(I)}{n})$ (proof omitted). In any case, Proposition~\ref{pro:largegap} shows that local discrepancy can be much smaller than global discrepancy. In contrast, global discrepancy cannot be much smaller than local discrepancy, as shown by Proposition~\ref{pro:half}. We now develop some machinery for proving Proposition~\ref{pro:half} for the case of quadratic discrepancy.

\begin{proposition}
\label{pro:2pixels}
Given a~2 by~1 image $I$ composed only of two adjacent pixels, $LD(I) = 2GD(I)$.
\end{proposition}

\begin{proof}
We prove the proposition for quadratic discrepancy. The proof for linear discrepancy is similar.

Let the two pixels be $p$ and $q$. Then $LD_2(I) = (x_p - x_q)^2$, whereas
$$GD_2(I) = \frac{1}{4}\left((x_p - x_p)^2 + (x_p - x_q)^2 + (x_q - x_p)^2 + (x_q - x_q)^2\right) =  \frac{1}{2}LD_2(I).$$
\end{proof}

Given an image $I$, an {\em equipartition} $P$ of $I$ partitions the set of its pixels into disjoint equal size subsets. The global discrepancy of an equipartition of $I$, denoted by $GD(P)$, is the average of the global discrepancies of its parts.

\begin{lemma}
\label{lem:equipartition}
For every image $I$ and every equipartition $P$ of $I$, $GD_2(P) \le GD_2(I)$.
\end{lemma}

\begin{proof}
For convenience of notation, let $n$ denote here (and only here) the total number of pixels in the image $I$, and suppose that the equipartition $P$ partitions the set of pixels into $n/d$ subsets, each with $d$ pixels. Number the pixels from~1 to $n$, with each subset occupying $d$ consecutive numbers. Consider now two $n$ by $n$ symmetric matrices. (These matrices are so called {\em Laplacian} matrices of graphs associated with the way discrepancy is being computed.) Matrix $A$ has $n-1$ along its diagonal, and all other entries are $-1$. Matrix $B$ is a block matrix with $n/d$ blocks of size $d$ along the diagonal. Each block has $d-1$ along its diagonal, and $-1$ elsewhere in the block. Outside the diagonal blocks, the matrix $B$ is all~0.

Let $x$ be the vector of values for the pixels of $I$. We think of $x$ as a column vector, and $x^T$ is its transposed row vector. Then $GD_2(I) = \frac{1}{n^2}x^TAx$, and the average discrepancy of the partition is $GD_2(P) = \frac{1}{dn}x^TBx$. Decompose $x$ into two components, $\alpha y + z$, where $y$ is the all~1 vector, $\alpha$ is the average value of $x$, and $z$ is a vector orthogonal to $y$. Observe that $y$ is an eigenvector of eigenvalue~0 both for $A$ and for $B$. Hence $GP_2(I) = \frac{1}{n^2}z^TAz$ and $GD_2(P) = \frac{1}{dn}z^TBz$. Observe that all eigenvalues of $A$, except for the unique~0 eigenvalue, have value $n$. Hence  $GP_2(I) = \frac{1}{n^2}z^TAz = \frac{1}{n}|z|^2$ (where $|z|$ is the norm of $z$). As for $B$, it has $n/d$ eigenvalues of~0, and each block contributes $d-1$ eigenvalues of value~$d$. Hence $GD_2(P) = \frac{1}{dn}z^TBz \le \frac{1}{n}|z|^2$. This establishes that $GD_2(P) \le GD_2(I)$, as desired.
\end{proof}

We can now prove the quadratic discrepancy part Proposition~\ref{pro:half}.

\begin{proof}
Suppose for simplicity that $n$ is even. Observe that the grid graph is nearly 4-regular. Add one edge to each row making that row into a cycle, and one edge to each column making the column into a cycle. Thus $2n$ edges are added, but they form only a $1/n$ fraction of the total number of edges, explaining the $(1/n)$ error term in the statement of Proposition~\ref{pro:half}. Consider now~4 different partitions of the grid (which by now is a torus), each into $n^2/2$ parts: $P_1$ takes all even pairs in the rows, $P_2$ takes all odd pairs in the rows, $P_3$ takes all even pairs in the columns, $P_4$ takes all odd pairs in the columns. By Lemma~\ref{lem:equipartition}, $GD_2(P_i) \le GD_2(I)$ for every $1 \le i \le 4$. Hence the average global discrepancy of a pair of adjacent pixels is at most $GD_2(I)$. Proposition~\ref{pro:2pixels} then implies that $LD_2(I) \le 2GD_2(I)$.
\end{proof}

\subsection{Lower bounds on discrepancy}

The following proposition shows that the bounds in Theorem~\ref{thm:local} are best possible.

\begin{proposition}
\label{pro:local}
For some environment $U$ the average local discrepancy satisfies $LD(U,k) \ge \frac{1}{k}$.
\end{proposition}

\begin{proof}
The following $U$ attains $LD(U,k) = \frac{1}{k}$. The cells of $U$ form a checkerboard pattern with alternating 0/1 values. In the scale $\ell=0$ the local discrepancy is~1 (regardless of the location of $I$), and in every other scale local discrepancy is~0. As $Pr[\ell = 0] = \frac{1}{k}$, the proposition follows.
\end{proof}

We now prove Proposition~\ref{pro:global} concerning global discrepancy.

\begin{proof}
Partition $U$ into {\em mega-cells} where a mega-cell is a $2^k$ by $2^k$ array of cells. Within a mega-cell, every cell has the same intensity. The mega-cells are arranged in a checkerboard pattern, with alternating 0/1 intensities.

The distribution $D$ selects a scale $\ell \in [0, k-1]$ uniformly at random. Observe that already when $\ell = k-1$, a random pixel has constant probability of being entirely contained in a mega cell, and this probability tends to~1 at an exponential rate as $\ell$ decreases. Moreover, when $n$ is even, as long as $\ell \ge k - \log n + 1$, exactly half the cells (not mega cells) contained in an image have intensity~1, and the other half has intensity~0. The combination of these two facts implies that roughly half the pixels of the image have value~1, and roughly half have value~0, giving $GD(I) \simeq \frac{1}{2}$. As this happens at roughly $\log n$ scales out of $k$ possible choices of scales, $E_{I \leftarrow D} GD(I) \simeq \frac{\log n}{2k}$, as desired.
\end{proof}

The bound in Proposition~\ref{pro:global} is nearly best possible, though this will not be proved in this manuscript, because we only need the direction of the inequality that is stated in the proposition.

We now prove Proposition~\ref{pro:sqrt} concerning linear discrepancy.

\begin{proof}
We shall not try to optimize the constant $c$ in the following proof.

In our proof it will be convenient to allow the intensities of cells to be in the range $[-2\sqrt{\log k}, 2\sqrt{\log k}]$, where for simplicity we assume that $2\sqrt{\log k}$ is integer. Clearly, by scaling intensities can be adjusted to lie in the range $[0,1]$, while losing a factor of $4\sqrt{\log k}$ in the value of $LD_1$.

Let the $N$ by $N$ environment (with $N > 2^k$ a power of~2) be such that the intensity of a cell $(i,j)$ depends only on $j$ but not on $i$. Specifically, the intensity of cell $(i,j)$ is computed as follows. Write $j$ in binary notation, but with $-1$ replacing~0. Consider only the $k$ least significant bits in this notation. This gives some string $r \in \{-1,1\}^k$. For cells that we refer to as {\em balanced} the intensity of the cell is simply the sum of bits in $r$. However, there are cells that we refer to as {\em extreme}. Those are the cells for which for some $q \le k$ the sum of the first $q$ bits in $r$ is either $-2\sqrt{\log k}$ or $2\sqrt{\log k}$. For these extreme cells their intensity is the value of the corresponding prefix (hence the maximum allowed absolute value for the intensity). By Kolmogorov's inequality for partial sums of independent $\pm 1$ random variables, at most one quarter of the cells are extreme.

Consider now a random pixel $p$ at an arbitrary scale $0 \le \ell < k$. For all the cells within it, the corresponding $r$ share the same $(k-\ell)$-prefix. Observe that when this prefix by itself is not extreme (namely, its sum of values never hits neither $-2\sqrt{\log k}$ nor $2\sqrt{\log k}$ -- this happens with probability at least $\frac{3}{4}$) then the value of the pixel (the average over all cells that it contains) is precisely the sum of values of the $(k-\ell)$-prefix. Of the four pixels adjacent to $p$, one of them is adjacent to it horizontally and agrees with it on an $(k - \ell -1)$-prefix and differs on bit $(k - \ell)$. The linear discrepancy between these two pixels is~2.

This implies that with probability at least $\frac{3}{16}$ the linear discrepancy is at least~2, which after scaling the intensities to lie in $[0,1]$ shows that $LD_1 \ge \frac{3}{32\sqrt{\log k}}$.
\end{proof}

\subsection{Proofs of main theorems}

Proof of Theorem~\ref{thm:local}.

\begin{proof}
In an image of scale $\ell$, the side length of a pixel is $2^{\ell}$ cells. Using distribution $D$, every scale $\ell \in \{0, \ldots, k-1\}$ is chosen with equal probability, and given a scale $\ell$, every two adjacent pixels of size $2^{\ell}$ are equally likely to be in the image. We need to prove that the expectation of the discrepancy $(x_p - x_q)^2$ of two adjacent pixels (chosen at random from an image chosen from distribution $D$) is at most $\frac{1}{k}$. It suffices to prove it for pairs of pixels adjacent horizontally, and by symmetry, the same proof will apply to pairs of pixels adjacent vertically. Hence for the rest of the proof, pixels are considered to be adjacent if and only if they are adjacent horizontally. We may envision a pair of adjacent pixels as a domino piece. We describe now a method of sampling uniformly at random a domino piece.

Consider a ``window" $W$ of $U$ with $2^k$ columns and $2^{k-1}$ rows. This window is equivalent to a domino piece of scale $k-1$. Subdivide each of its pixels of scale $k-1$ into four pixels of scale $k-2$. These pixels are arranged as two domino pieces of scale $k-2$. Continue subdividing recursively, where for every $\ell \ge 1$, every pixel of scale $\ell$ gives two domino pieces of scale $\ell - 1$. Hence in scale $\ell$ there are $4^{k-1-\ell}$ disjoint domino pieces.
Now to sample a random domino piece, choose $W$ at random, choose a scale $\ell$ uniformly at random, and within $W$ choose a domino piece of scale $\ell$ uniformly at random.

To compute the discrepancy of a domino piece of scale $\ell$, one needs first to average the value of its left pixel $p$ (by summing all cells and dividing by $2^{2\ell}$) getting a value $x_p$, to average the value of its right pixel $q$ getting a value $x_q$, and compute $(x_p - x_q)^2$.

Let $W_{\ell}$ denote the set of domino pieces of scale $\ell$ in $W$. Let $LD(W)$ denote the weighted average local discrepancy (over horizontal pairs) in $W$, where the weights are such that each scale is equally likely to be chosen. We have:

\begin{equation}\label{eq:LD}
LD(W) = \frac{1}{k}\sum_{\ell = 0}^{k-1} 4^{\ell+1 - k}\sum_{(p,q) \in W_{\ell}} (x_p - x_q)^2
\end{equation}

The intensities of cells in $W$ is a function $I$ from $2^{k-1} \times 2^k$ cells of $W$ to $[0,1]$. Denoting cells by $c$, the average intensity $2^{1 - 2k} \sum_{c\in W} I(c)$ will be denoted by $\mu$, and the average of the squares of the intensities $2^{1 - 2k} \sum_{c\in W} (I(c))^2$ will be denoted by $w^2$.
We now represent the function $I$ in an orthonormal basis that is very much related to the Haar basis, though not identical to it. The number of basis vectors needs to be $2^{2k-1}$ (matching the number of cells in $W$), but we shall specify only some of the basis vectors. The set of basis vectors that we specify will be referred to as the {\em domino partial basis}. One basis vector $v_0$ has value $\frac{1}{2^{k}\sqrt{2}}$ on all cells of $W$. In addition, each domino piece in $W$ represents a basis vector as follows. Given the scale $\ell$ of the domino piece, in its left pixel (composed of $2^{2\ell}$ cells), each cell has value $\frac{1}{2^{\ell}\sqrt{2}}$, each cell in the right pixel has value $-\frac{1}{2^{\ell}\sqrt{2}}$, and the cells not covered the domino piece have value~0.  Hence the norm of every vector in the domino partial basis is~1, and every two vectors are orthogonal.

The inner product of $I$ with a basis vector that corresponds to a domino piece $(p,q)$ at scale $\ell$ is precisely $\frac{2^{\ell}}{\sqrt{2}}(x_p - x_q)$. This is the coefficient of the function $I$ according to the basis vector corresponding to the domino piece.  The square of this coefficient is $2^{2\ell-1}(x_p - x_q)^2$. For $v_0$, the squared value of the coefficient can readily be seen to be $2^{2k-1}\mu^2$. The sum of squares of all coefficients is at most the square of the norm of $I$ (if we had a complete basis, they would be equal, by Parseval's identity), and hence:

\begin{equation}\label{eq:parseval}
2^{2k-1}\mu^2 + \sum_{\ell = 0}^{k-1} \sum_{(p,q) \in W_{\ell}} 2^{2\ell-1}(x_p - x_q)^2 \le \sum_{(i,j)\in W} (W_{i,j})^2 = 2^{2k-1} w^2
\end{equation}

Dividing both sides of Equation~(\ref{eq:parseval}) by $2^{2k-1}$, we obtain

\begin{equation}\label{eq:parseval1}
\mu^2 + \sum_{\ell = 0}^{k-1} \sum_{(p,q) \in W_{\ell}} 4^{\ell-k}(x_p - x_q)^2 \le w^2
\end{equation}

Combining Equation~(\ref{eq:parseval1}) with~(\ref{eq:LD}) we obtain that $LD(W) \le \frac{4}{k}(w^2 - \mu^2)$. As the intensities are in the range $[0,1]$, necessarily $w^2 \le \mu$. The expression $\mu - \mu^2$ is maximized when $\mu = \frac{1}{2}$, and then it evaluates to $\frac{1}{4}$. Hence $LD(W) \le \frac{1}{k}$, as desired.
\end{proof}

For the proof of Theorem~\ref{thm:rho} we use the following notation. Let $x_{\ell} = LD_2(\ell)$ for $0 \le \ell \le k-1$ and $x_{\ell} = 0$ for $\ell \ge k$.

\begin{lemma}
\label{lem:local}
For every $0 \le \ell \le k-1$, $2GD_2(\ell) \ge \sum_{\ell' = \ell}^{\ell + \log n - 1} x_{\ell'}$.
\end{lemma}

\begin{proof}
The proof of Lemma~\ref{lem:local} is implicit in our proof of Theorem~\ref{thm:local}. Consider a random window $W$ of $U$ with $n2^{\ell - 1}$ rows and $n2^{\ell}$ columns. It can be thought of as half an image $I$ at scale $\ell$, and being half an image, Lemma~\ref{lem:equipartition} implies that the expectation over choice of random $W$ satisfies $E[GD_2(W)] \le E[GD_2(I)] = GD_2(\ell)$. The proof of Theorem~\ref{thm:local} implies that $\sum_{\ell' = \ell}^{\ell + \log n - 1} x_{\ell'} \le 4E[w^2 - \mu^2]$, where $\mu$ is the average value of a pixel in $W$ and $w^2$ is the average squared value. As $GD_2[W] = 2(w^2 - \mu^2)$, the lemma follows.
\end{proof}


We now prove Theorem~\ref{thm:rho}.

\begin{proof}
Observe that convexity of the function $x^2$ implies that $LD_1(\ell) \le \sqrt{x_{\ell}}$. Hence to prove Theorem~\ref{thm:rho} we shall bound the maximum possible value of $\sum_{\ell=0}^{k-1} \sqrt{x_{\ell}}$. As $x_{\ell} = 0$ for $\ell \ge k$, this is the same as bounding the maximum possible value of $\sum_{\ell \ge 0} \sqrt{x_{\ell}}$. Relaxing the constraint that $x_{\ell} = 0$ for $\ell \ge k$, we get the following mathematical program.

{\em Maximize} $\sum_{i \ge 0} \sqrt{x_i}$ {\em subject to:}

\begin{enumerate}

\item $x_i \ge 0$.

\item $\sum x_i \le 1$.

\item $2\alpha x_{i} \ge \sum_{j = i}^{i + \log n - 1} x_j$.

\end{enumerate}

Constraint~2 is a consequence of Theorem~\ref{thm:local}. Constraint~3 is a consequence of Lemma~\ref{lem:local} together with the premise of Theorem~\ref{thm:rho}.

Consider a feasible (not necessarily optimal) solution to the above mathematical program of the form $x_i = (1-p)p^i$, for some $0 < p < 1$. Then constraint~1 is necessarily satisfied. Constraint~2 is satisfied with equality because $\sum_{i \ge 0} (1-p)p^i = (1-p)\sum_{i \ge 0} p^i = \frac{1-p}{1-p} = 1$. As for Constraint~3, we require that $2\alpha (1-p)p^i \ge \sum_{j = i}^{i + \log n - 1} (1-p)p^j$. Dividing both sides by $(1-p)p^i$ we get an upper bound on the maximum possible value of $p$, implied by the inequality $2\alpha \ge \sum_{j=0}^{\log n + 1} p^j = \frac{1 - p^{\log n}}{1-p}$.

If $x_i$ is of the form $(1-p)p^i$, then in order to maximize $\sum_{i \ge 0} \sqrt{x_i}$ we need to choose $p$ as large as possible. This follows because

$$\sum_{i \ge 0} \sqrt{x_i} = \sum_{i \ge 0} \sqrt{1-p}(\sqrt{p})^i = \frac{\sqrt{1 - p}}{1 - \sqrt{p}} = \frac{1 + \sqrt{p}}{\sqrt{1-p}}$$
is increasing with $p$.

Recall that $p$ needs to satisfy the constraint:

$$\frac{1 - p^{\log n}}{1-p} \le 2\alpha$$

In particular, when $n$ tends to infinity, we have that $p \le 1 - \frac{1}{2\alpha}$.
Under the solution $p = 1 - \frac{1}{2\alpha}$ the value of the objective function of the mathematical program has the following simple form:

$$\sum_{i \ge 0} \sqrt{x_i} = \frac{1 + \sqrt{p}}{\sqrt{1-p}} = \sqrt{2\alpha} + \sqrt{2\alpha - 1}$$

It remains to show that the solution $x_i = (1-p)p^i$ is not only feasible but also optimal. Hence fix $\alpha > 1$ and integer $\log n > 2\alpha$ and let $0 < p < 1$ be the solution of $\frac{1 - p^{\log n}}{1-p} = 2\alpha$. (The inequality $\log n > 2\alpha$ is required in order to ensure the existence of such a $p$.)

Consider an optimal solution $X = x_0, x_1, \ldots$, and for the sake of contradiction suppose that there is some $i$ (we take the smallest one) for which $x_i \not= (1-p)p^i$. We consider two cases.

\begin{enumerate}

\item
$x_i < (1-p)p^i$. Let $i \le j \le i + \log n - 1$ be largest such that $x_j < (1-p)p^j$. Constraint~3 implies that necessarily $j > i$. Likewise, Constraint~3 implies that $\sum_{\ell = i+1}^j x_{\ell} < \sum_{\ell=i+1}^j (1-p)p^{\ell}$. The same argument can be repeated with $j$ replacing $i$, and thereafter repeated indefinitely.
By minimality of $i$ we have that $\sum_{\ell = 0}^{i-1} x_{\ell} = \sum_{\ell = 0}^{i-1} (1-p)p^{\ell}$. Hence we have that $\sum_{\ell \ge 0} x_{\ell} < \sum (1-p)p^{\ell} = 1$. This means that in the solution $X$ can increase $x_i$ (in fact, at least up to $(1-p)p^i$) without violating any of the constraints, thus contradicting the optimality of $X$.

\item
$x_i > (1-p)p^i$. An argument analogous to Case~1 above implies that it cannot be that for every $j$ Constraint~3 is attained with equality, as then Constraint~2 will be violated. Let $j$ be the smallest index for which there is slackness in Constraint~3,  and let $\epsilon_1$ be the amount of slackness. Denote $\epsilon_2 = x_j - x_{j+1}$ and suppose that $\epsilon_2 > 0$. In this case, modify the solution $X$ to a new solution $X'$ in which $x_j$ is replaced by $x'_j = x_j - \frac{1}{2}\min[\epsilon_1,\epsilon_2]$ and $x_{j+1}$ is replaced by $x'_{j+1} = x_{j+1} + \frac{1}{2}\min[\epsilon_1,\epsilon_2]$. One can easily verify that $X'$ is feasible and gives a higher value than $X$ does for the objective function (due to concavity of $\sqrt{x}$). This contradicts the assumed optimality of $X$.

It remains to deal with the case that $\epsilon_2 \le 0$. Below we establish that in this case there is some other index $q > j$ such that Constraint~3 has slackness for $x_q$, and moreover, $x_q > x_{q+1}$. Then the above argument can be applied with $q$ replacing $j$, completing the proof.

Observe that there are only finitely many indices $q$ with $x_q \ge x_j$ (because $\sum x_i \le 1$). Let $q$ be the largest index such that $x_q \ge x_j$. By our assumption that $\epsilon_2 \le 0$ we have that $q > j$. Clearly $x_q > x_{q+1}$. We now show that Constraint~3 has slackness for $x_q$. There are two cases to consider.

\begin{enumerate}

\item For all $j \le i \le q$ it holds that $x_i \ge x_j$. In this case $q < j + \log n - 1$, because Constraint~3 (together with $\log n > 2\alpha$) implies that the average value of $x_i$ for $j \le i \le j+\log n - 1$ is less than $x_j$. It follows that the inequalities implied by Constraint~3, one for $x_j$ and one for $x_q$, overlap in some terms on the right-hand side. Moreover, for all the terms in which they differ, the right hand side for $x_j$ has strictly higher value (every term at least $x_j$) than for $x_q$ (every term strictly smaller than $x_j$). Since $x_q \ge x_j$, there must be slackness for $x_q$.

\item For some $j < i < q$ it holds that $x_i < x_j$. Let $i < q$ be the largest such index. Then repeat the argument above with the inequalities implied by Constraint~3, one for $x_i$ (instead of $x_j$) and one for $x_q$.

\end{enumerate}

\end{enumerate}

\end{proof}

\section{Discussion}
\label{sec:discussion}

One may think of our work as distinguishing between three concepts.

\begin{enumerate}

\item An array of pixels.

\item An image as defined in our abstract model. It depicts a portion of an environment, and may do so in one of several scales. No assumptions are made on the nature of the environment.

\item A natural image. The environment depicted needs to adhere to physical realities of our world, and the selection process of images may be biased, based on the goals of the person taking these images.

\end{enumerate}

The three main principles that underlie our probabilistic model of images are the following:

\begin{itemize}

\item The model assumes nothing about the nature of the environment $U$. As our results are positive (showing some level of smoothness), this aspect strengthens the applicability of our results.

\item There is no single scale in which a large fraction of the images are taken. If there was such a scale, then $U$ can be arranged to have large local discrepancy at this scale (e.g., a checkerboard pattern), and on average images would not be smooth.

\item The location of the image is chosen independently of the content of $U$ -- for a given scale, there is no correlation between the smoothness of $U$ at a certain location and the probability that the image is taken at this location. 

\end{itemize}

We showed that a key statistical property associated with natural images, that of smoothness, already manifests itself to some extent in the abstract model for images. Our study is quantitative, and our quantitative results uncover rather subtle and perhaps counter-intuitive effects. Let us recap one of our conclusions. Arguably, noticeable local correlation in an image (namely, having quadratic local discrepancy that is small compared to the quadratic global discrepancy) is by itself an indication for smoothness. Theorem~\ref{thm:rho} (contrasted with proposition~\ref{pro:sqrt}) shows that the {\em absence} of local correlation (setting $\alpha$ close to~1) leads to improved upper bounds on the expected linear local discrepancy of random images.

Given quantitative values of smoothness of natural images, our work may allow one to assess how much of this value should be attributed already to the abstract image model, and then only the residual smoothness needs to be explained by properties of the natural world.

Our results become more significant as the number $k$ of scales grows. In natural images, due to physical constraints of the real world, $k$ cannot grow indefinitely, and hence we attempted to present our results not only in an asymptotic sense (e.g., $O(1/k)$), but also to provide explicit bounds on the leading constants involved. In particular, the premise of Theorem~\ref{thm:rho} was chosen in a way that would keep these constants small. The proof technique of Theorem~\ref{thm:rho} (using a linear program to upper bound the linear discrepancy) is versatile enough to extend to weaker premises, at the cost of resulting in higher leading constants in the $O(1/k)$ upper bound.

\subsection*{Acknowledgements}

The author thanks Ronen Basri, Anat Levin and Boaz Nadler for helpful discussions on natural image statistics.


\begin{thebibliography}{22}



\bibitem{AGM1999} Luis Alvarez,
Yann Gousseau, Jean-Michel Morel. The Size of Objects in Natural and Artificial Images.
Advances in Imaging and Electron Physics, Volume 111, 1999, Pages 167-–242.

\bibitem{book} Aapo Hyvarinen, Jarmo Hurri, Patrik O. Hoyer.
Natural Image Statistics:
A probabilistic approach to early
computational vision.
Springer-Verlag, 2009.

\bibitem{wavelet} Ingrid Daubechies. Ten Lectures on Wavelets. SIAM, 1992.

\bibitem{FKT} Uriel Feige, Tomer Koren, Moshe Tennenholtz. Chasing Ghosts: Competing with Stateful Policies. Manuscript, 2014.


\bibitem{mumford} D. Mumford and B. Gidas. Stochastic models for generic images. Quarterly of Applied Mathematics, 54(1):85–111, 2001.

\bibitem{ruderman} D. Ruderman. Origins of scaling in natural images. Vision Res., Vol. 37, No. 23, pp. 3385--3398, 1997.

\bibitem{forest} Daniel L. Ruderman,  William Bialek. Statistics of Natural Images: Scaling in the Woods. Physical Review Letters, 73(6), 814–-817, 1994.
    
\bibitem{irani} Maria Zontak, Michal Irani: Internal statistics of a single natural image. CVPR 2011: 977--984.


\end{thebibliography}
\end{document}